\documentclass{article}
\usepackage{iclr2015} 
\usepackage{amsmath}
\usepackage{amsthm}
\usepackage{amsfonts}
\usepackage{times}
\usepackage{graphicx} 
\usepackage{subfigure} 
\usepackage{natbib}

\newcommand{\real}{\mathbb{R}}
\newtheorem{theorem}{Theorem}
\newtheorem{lemma}{Lemma}

\title{Crypto-Nets: Neural Networks over Encrypted Data}
\author{Pengtao Xie\thanks{The work was done while this author was visiting Microsoft.} \\
School of Computer Science\\
Carnegie Mellon University \\
Pittsburgh, USA\\
\And
Misha Bilenko \& Tom Finley   \\
Cloud Computing Division \\
Microsoft\\
Redmond, USA\\
\And
Ran Gilad-Bachrach \& Kristin Lauter \& Michael Naehrig \\
Microsoft Research\\
Microsoft\\
Redmond, USA
}

\iclrconference

\begin{document} 

\maketitle

\begin{abstract}
The problem we address is the following: how can a user employ a predictive model that is held by a third party, without compromising private information. For example, a hospital may wish to use a cloud service to predict the readmission risk of a patient. However, due to regulations, the patient's medical files cannot be revealed. The goal is to make an inference using the model, without jeopardizing the accuracy of the prediction or the privacy of the data. 

To achieve high accuracy, we use neural networks, which have been shown to outperform other learning models for many tasks. To achieve the privacy requirements, we use homomorphic encryption in the following protocol: the data owner encrypts the data and sends the ciphertexts to the third party to obtain a prediction from a trained model. The model operates on these ciphertexts and sends back the encrypted prediction. In this protocol, not only the data remains private, even the values predicted are available only to the data owner.

Using homomorphic encryption and modifications to the activation functions and training algorithms of neural networks, we present crypto-nets and prove that they can be constructed and may be feasible. This method paves the way to build a secure cloud-based neural network prediction services without invading users' privacy.
\end{abstract}

\section{Introduction}

Recently, many efforts have been devoted to cloud machine learning (CML), where machine learning (ML) services are running on commercial providers' infrastructure. Examples include Microsoft Azure Machine Learning\footnote{http://azure.microsoft.com/en-us/services/machine-learning/}, Google Prediction API\footnote{https://developers.google.com/prediction/},  GraphLab\footnote{http://graphlab.com/} and Ersatz Labs\footnote{http://www.ersatzlabs.com/}, to name a few. CML allows training and deploying models on cloud servers. Once deployed users can use these models to make predictions without having to worry about maintaining the service and the models. Moreover, it allows the model owner to be paid for every prediction being made by the model. In a broader sense, it enables a model of Machine Learning as a Service (MLaaS), where there is a separation between the data owner, the model owner and the compute provider (the cloud).

Despite the attractive benefits provided by MLaaS, it suffers from a severe problem, namely the invasion of the security and privacy of users' data. Traditional ML solutions require access to the raw data, which creates a potential security and privacy risk. In some cases, for example that of medical data, regulations may make these usage patterns illegal.  Therefore, the goal of this work is to enable data owners to use MLaaS without exposing their data. 

This problem has been addressed before by \cite{graepel2013ml}. They proposed to perform machine learning on encrypted data utilizing  homomorphic encryption. A homomorphic encryption scheme \citep{rivest1978data} allows a certain computation to be performed on the encrypted data by manipulating the corresponding ciphertexts without the need to decrypt them first. A fully homomorphic encryption scheme \citep{gentry2009fully} allows arbitrary operations over encrypted data and therefore, any function can be computed. However, fully homomorphic encryption schemes are still too inefficient for practical use. One way to obtain better efficiency is to only use so-called somewhat homomorphic schemes that only allow the evaluation of functions up to a certain complexity.  Such schemes are often the cores of corresponding fully homomorphic encryption schemes. They usually provide operations corresponding to addition and multiplication of encrypted integer values, and therefore, are suitable to evaluate polynomial functions up to a certain maximal degree. The required degree of the polynomial function along with the desired security level determines the scheme parameters and thus has great implications on the size of the ciphertext as well as the computational complexity of the cryptographic operations. Therefore, \cite{graepel2013ml} suggested using linear or other low degree models. While this method preserves the privacy and security of the data, it does not allow for highly accurate predictions since linear models cannot compete with the state-of-the-art in terms of accuracies on problems such as object recognition in image or speech data.

\begin{figure}
\begin{center}
\includegraphics[width=0.9\columnwidth]{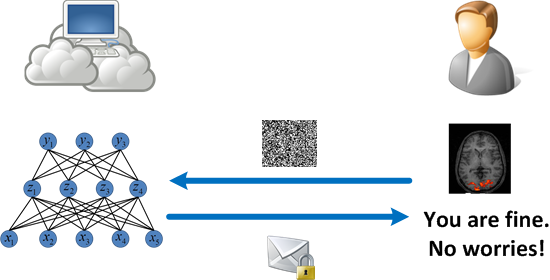}
\caption{Privacy preserving neural network prediction on encrypted data. The cloud holds a trained neural network model which provides a prediction service. To protect their privacy, users encrypt their data and upload the ciphertexts to the cloud. The cloud makes neural network predictions over the encrypted data and obtains results in encrypted form. The cloud sends the encrypted prediction results back to the users. Users  locally decrypt and get the plaintext results. The cloud never learns any information about the users' data and the prediction results, while the results are guaranteed to be correct.}
\end{center}
\label{fig:nnpe}
\end{figure}
In this paper, we investigate how to perform neural network prediction on encrypted data. A neural network is a nonlinear machine learning model with large model capacity. It has achieved great success in speech recognition, image classification and natural language processing. Figure \ref{fig:nnpe} illustrates the scenario of making secure predictions on encrypted data with a neural network. On the cloud side, there is a neural network model trained on plaintext data. For example, let us assume that the trained neural network takes medical images and predicts the likelihood of a pathology (disease). A user possesses a medical image and wants to use the neural network model in the cloud to predict whether he has the disease. Meanwhile, the user does not want the image to be seen by the cloud, because it may leak his health conditions. The user encrypts the image into a ciphertext and sends the ciphertext to the cloud. The cloud service evaluates the neural network prediction by operating on the ciphertext only and produces a prediction result in encrypted form that the cloud cannot decipher. The encrypted result is sent back to the user, who decrypts locally and retrieves the result as readable plaintext. 
In this process, both the input image and the output prediction are held in encrypted form. The cloud does not learn any information about the users' input data and the prediction result. Thereby, confidentiality of the user's data and prediction results are guaranteed. %On the other hand, the prediction on encrypted data is nearly the same as that on plaintext data. There is no significant sacrifice of predictive performance.

The main challenge in realizing this solution is the fact that the commonly used activation functions in neural networks are not in polynomial form. This includes functions such as the sigmoid and rectified linear functions. We first show that from theoretical point of view, since these functions are continuous, they can be approximated by polynomials and therefore, the entire computation can be thought of as applying a polynomial to the data. We also discuss ways to minimize the degree of these polynomials such that the time to compute will remain feasible. We call this type of neural networks crypto-nets.

\section{Related Work}
Using Homomorphic Encryption (HE) to do machine learning and statistical analysis on encrypted data has been investigated in \citep{bos2014private, bostmachine2014,graepel2013ml, lauter2014,nikolaenko2013privacy,nikolaenko2013privacy2,wu2012using}. These works have studied how to do HE-based privacy-preserving training or prediction of linear regression \citep{nikolaenko2013privacy2,wu2012using}, linear classifiers \citep{bos2014private,bostmachine2014,graepel2013ml}, decision trees \citep{bostmachine2014}, matrix factorization \citep{nikolaenko2013privacy}. As far as we know, ours is the first work to show how to apply neural networks to encrypted data and therefore allow the use of models that have been shown to be very accurate. 

\cite{orlandi2007oblivious} suggested a scheme for using homomorphic encryption with neural networks. They suggest solving the problem of non-linear activation functions by creating an interactive protocol between the data owner and the model owner. In a nut-shell, every non-linear transformation is computed by the data-owner: the model sends the input to the non-linear transformation in encrypted form to the data owner who decrypts the message, applies the transformation, encrypts the result and sends it back. Unfortunately, this interaction requires large latencies and increases the complexity on the data owner side, effectively making it impractical. Moreover, it leaks information about the model. Therefore, \cite{orlandi2007oblivious} had to introduce safety mechanisms, such as random order of execution, to mitigate this issue.  In comparison, the procedure we introduce does not require complicated communication schemes: the data owner encrypts the data and sends it. The model does its computation and sends back the (encrypted) prediction. Therefore, it allows for asynchronous communication and it does not leak unnecessary information about the model. 

Another line of work focuses on differential privacy \citep{chaudhuri2011differentially, Duchi2012privacy, dwork2008differential, smith2011privacy, wasserman2010statistical}. Differential privacy aims at allowing to gather statistics from a database without revealing information about individual records. However, this method is not suitable for privacy-preserving prediction  since by its nature, in the inference phase, a single record is being used and therefore fully exposed. Moreover, the method proposed here provides a  much higher level of security. For example, not only the row records are not exposed, even the predicted value is not accessible to any party except the data owner since it is encrypted, not even to the cloud service that computed it, since it is encrypted.

\section{Homomorphic Encryption}
\label{sec:he}

A Homomorphic Encryption (HE) scheme \citep{rivest1978data} preserves some structure of the original message space. Here, we assume that it provides methods to add and multiply encrypted messages and therefore preserves the message space ring structure. We also assume that it can be used to operate on the ring of integers. In that case, messages are integers and the scheme preserves the ability to perform additions and multiplications of such integers.  

For our purpose, a (secret key) HE scheme consists of four algorithms: encryption ($E$), decryption ($D$), addition ($\oplus$) and multiplication ($\otimes$). The encryption algorithm takes as input a message and a secret key $k$. We denote the dependence on the key by $E_k$, but will drop the subscript later when use is clear from the context. The decryption takes as input an element from the ciphertext space and a key, while the algorithms $\oplus$ and $\otimes$ do not depend on the secret key and only take two ciphertexts as input.  Let $m_1$ and $m_2$ be integer messages and let $k$ be a secret key.  Then the above algorithms have the following properties: 
\begin{enumerate}
\item Given $E_k(m_1)$, it is computationally infeasible to compute $m_1$ without the private key $k$.
\item It holds that $m_1 = D_k(E_k(m_1))$.
\item It holds that $m_1 + m_2\ = D_k\left(E_k\left(m_1\right) \oplus E_k\left(m_2\right)\right)$.
\item It holds that $m_1 \times m_2 = D_k\left(E_k\left(m_1\right) \otimes E_k\left(m_2\right)\right)$.
\item The algorithms $\oplus$ and $\otimes$ do not use the secret key used for encryption.
\end{enumerate}

Furthermore, we require that the scheme can evaluate the algorithms $\otimes$ and $\oplus$ repeatedly for a certain number of times, while decryption still gives the correct result. More precisely, let $P$ be a polynomial on $n$ variables of degree at most $d$. Denote by $\tilde{P}$ the function on input $n$ ciphertexts, which is given by replacing each addition in $P$ by the algorithm $\oplus$ and each multiplication by $\otimes$.  Let $m_1,\dots,m_n$ be messages. Then the above algorithms satisfy the following property:
$$
P(m_1,\dots,m_n) = D(\tilde{P}(E(m_1),\dots,E(m_n))).
$$
This means that our HE scheme allows to compute any degree-bounded polynomial function $P$ as above over encrypted messages without decrypting them first. 

\cite{gentry2009fully} was the first to show that it is possible to construct a Fully Homomorphic Encryption (FHE) scheme, which means that there is no limit on the degree of the polynomial $P$ above. In theory, this allows to evaluate arbitrary computations (since any computation can be written as a binary polynomial in terms of binary addition and multiplication on the single bits of the input). Even though there has been great progress in making FHE schemes more efficient and secure (see, for example, \cite{brakerski2014efficient}), this approach is currently not feasible for practical applications. Efficiency can be increased by restricting to somewhat homomorphic schemes and by operating on integers instead of bits, see \cite{naehrig2011ca}. With this approach, both the computational complexity and the length of ciphertexts increase with the number of desired operations performed on the encrypted data in order to guarantee correct decryption after polynomial evaluation. While this increase is benign when increasing the number of additions, it is more significant when adding multiplications. Thus, a solution that builds upon these encryption schemes has to be restricted to computing low degree polynomials.

\section{Polynomial Approximation to Neural Networks}
\label{sec:poly}

From the discussion above, in Section~\ref{sec:he} we conclude that certain polynomial functions can be computed over encrypted data given that their degree is not too large. However, activation functions such as sigmoids and rectified-linear functions are not polynomials and the same applies to other, commonly used non-linear transformations in neural-networks such as max pooling. Nevertheless, since all these functions are continuous, the results, that is the neural net, viewed as a function, is a continuous function. If the domain, that is the input space, is a compact set, then from the Stone-Weierstrass theorem \citep{Stone1948} it follows that it can be approximated uniformly by polynomials. We will begin the discussion with the inference case, therefore we assume that the neural network has already been trained and the goal is to apply it to encrypted data.

\begin{lemma}\label{lemma:polyApprox}
Let $N$ be a neural network in which all non-linear transformations are continuous. Let $X\subset\real^n$ be the domain on which $N$ acts and assume that $X$ is compact, then for every $\epsilon>0$ there exists a polynomial $P$ such that 
$$ \sup_{x\in X}\left\Vert N(x)-P(x)\right\Vert < \epsilon~~~.$$ 
\end{lemma}
\begin{proof}
The function $N$ is constructed by compositions, additions and multiplications over the inputs and the non-linear transformations. Since compositions, additions and multiplications of continuous functions are continuous, the function $N$ is continuous. Since $N$ is a continuous function over a compact space and since the set of polynomials is an Algebra that separates points it follows from the Stone-Weierstrass theorem \citep{Stone1948} that there exists a polynomial $P$ such that $$ \sup_{x\in X}\left\Vert N(x)-P(x)\right\Vert < \epsilon~~~.$$ 
\end{proof}

Note that the assumption that the non-linearity is continuous is very mild since the back propagation algorithm used for learning neural networks assumes the existence of a gradient or a sub-gradient to these functions which implies continuity.

\begin{theorem}
Let $(E,D)$ be the encryption and decryption functions of a HE system. Let $N$ be a neural network in which all non-linear transformations are continuous. Let $X\subset\real^n$ be the domain on which $N$ acts and assume that $X$ is compact, then for every $\epsilon>0$ there exists a function $N^\prime$ such that 
$$ \sup_{x\in X}\left\Vert N(x) - D\left(N^\prime\left(E(x)\right)\right)\right\Vert < \epsilon$$
\end{theorem}\label{theorem:inference}
Note that for a vector $x=(x_1,\ldots,x_n)$ we use the notation $E(x)=\left(E(x_1),\ldots,E(x_n)\right)$ and $D(x)=\left(D(x_1),\ldots,D(x_n)\right)$
\begin{proof}
From Lemma~\ref{lemma:polyApprox} it follows that there exists a polynomial $P$ such that $\sup_{x\in X}\left\Vert N(x)-P(x)\right\Vert < \epsilon$. $N^\prime$ can be constructed from $P$ by replacing the addition and multiplications by the appropriate HE functions ($\oplus,\otimes$) and by replacing the constants in the polynomials by the encrypted versions of these constants. This can be done by accessing only the public encryption function $E$.
\end{proof}

Theorem~\ref{theorem:learning-polynomial} shows that an existing neural network can be applied to encrypted data. This is done by a two stage process: first the network is approximated by a polynomial and next this polynomial is "encrypted". Next we look at the learning process. The common way to learn a neural network is using back-propagation. This is a gradient descent type algorithm. That requires computing the derivative of the neural network with respect to the weights. If the neural network is a polynomial function (or is approximated by one) then the derivatives are polynomials as well and hence can be computed over encrypted data. However, some further restrictions are needed in some cases.

\begin{theorem}\label{theorem:learning-polynomial}
Fix the topology of a neural network and assume that all the non-linear transformations and the loss function are polynomials. Then the back propagation algorithm can be converted to work on encrypted data such that it will learn the encrypted version of the coefficients that the back propagation will learn on plain data.
\end{theorem}
\begin{proof} Since all transformations are polynomials then the function that the neural network computes is a polynomial. Since the loss function is polynomial as well it implies that the gradient is a polynomial too and therefore it can be computed over encrypted data.
\end{proof}

Theorem~\ref{theorem:learning-polynomial} suggests the following method for learning with encrypted data: first approximate all non-linear transformations with polynomials which will result in a polynomial network that can be learned exactly even when the data is encrypted. Note, however that when learning over encrypted data the results, that is the weights, are encrypted and if the learning algorithm does not have access to the secret key for use in the decryption function $D$ it will not be able to know what these coefficients are.

Another approach for learning with encrypted data is to approximate the back-propagation step with polynomials as illustrated by the following theorem.

\begin{theorem}
Assume that the domain of the network $X\subset\real^n$ is compact. Assume that the non-linear transformations and the loss function have continuous derivatives. Let $L$ be the back-propagation learning algorithm that maps a sample of size $T$ to the weight vector of the neural net. For every $\epsilon>0$ there exists a learning algorithm $L^\prime$ such that if $L$ learns the weights $w_1,\ldots,w_k$ from the sample $S$ of size $T$ then $L^\prime$ learns the weights $E(w^\prime_1),\ldots,E(w^\prime_k)$ form the sample $E(S)$ such that $\forall j,~~\vert w_j - w^\prime_j \vert < \epsilon$.
\end{theorem}
The proof is very similar to previous proofs and therefore we skip details.
\begin{proof} 
 The learning algorithm $L$ is made of addition, multiplication and compositions of the constants, non-linear transformations, the loss function and their gradients. According to the assumption of this theorem, all these functions are continuous and therefore $L$ is a continuous function over a compact space which can be approximated by a polynomial. The algorithm $L^\prime$ is this polynomial approximation of $L$ after all constants have been replaced by their encrypted versions and additions and multiplications have been replaced by the $\oplus, \otimes$ operations.
\end{proof}

\section{Practical consideration}\label{sec:practical}
In Section~\ref{sec:poly} we have shown that it is possible to learn neural networks over encrypted data and to apply neural networks to encrypted data. However, some scenarios may be infeasible due to excessive computational complexity. In this section we discuss practical considerations in more details.

 While HE schemes allow the evaluation of polynomial functions, these computations are much slower than computations done on plain data. Furthermore, in current implementations of HE, high degree polynomials are slower to compute than lower degree polynomials. The reason for that, in a nut shell, is that as part of the encryption process some random noise is added to the message. When adding two numbers via the $\oplus$ operation, the noise in the resulting ciphertext increases linearly with respect to the number of additions, however, when multiplying, the noise grows super--linearly. For an FHE scheme, when the noise size reaches a certain level, a time consuming cleaning process is performed which slows down the entire process. For HE schemes as the one considered in this work, the parameters of the scheme have to be chosen to accommodate the noise growth incurred by the desired computation. A higher complexity requires larger parameters, which leads to slower execution of the algorithms. Therefore, special considerations should be taken to approximate the neural network with polynomials with the lowest degree possible.

  Let $N$ be a neural network with $l$ layers. If the composition of the activation function and pooling functions in each layer is approximated by a polynomial of degree $d$ then the polynomial approximation of $N$ will be a polynomial of degree $d^l$ since when composing polynomials, the degrees of the polynomials multiply. Therefore, in order to end up with low degree polynomials, we need both $d$ and $l$ to be small. Minimizing $d$, the degree of the polynomial approximation to non-linear functions, is a standard exercise in approximation theory. Tools, such as, Chebyshev polynomials, can be used to find optimal or close to optimal approximations. Even more significant is minimizing the number of layers $l$. This goes against the current trend of learning deep neural networks. However, recent work on model compression ~\citep{buciluǎ2006model,BaCaruana14} show that deep nets can be closely approximated by shallow nets (1-2 hidden layers). These studies suggest that the success of deep nets might be due to better optimization and not necessarily from the kind of function space spanned by deep nets. Therefore, once you have a deep net, you can use it to train a shallow net by labeling a large set of unlabeled instances. This procedure converts deep nets to shallow, but wider, nets. In terms of polynomials, the deep nets convert to high degree polynomials while the shallow but wide nets convert to low degree polynomials with many monomials. Hence this conversions results in polynomials that are faster to execute on encrypted data.

  While inference using crypto-nets may be feasible, learning is a more difficult to scale tasks. Training neural networks is a computational intensive task. Even without encryption, high throughput computing units such as GPUs or multi-node clusters are needed to make learning neural nets feasible on large datasets \citep{dean2012large, coates2013deep}. Furthermore, assuming, as before, that the neural network has $l$ layers such that each layer is approximated by a polynomial of degree $d$ results in the neural network of degree $d^l$. The gradient of this network, with respect to the weight vector, is a polynomial of the same degree. To make gradient step, the gradient of the loss function is computed, if the $L2$ loss function is used then the gradient of the loss function will be  a polynomial of degree $d^{2l}$. However, this term does not take into account the constants of the polynomial which, when learned, are functions of the data from previous iterations. When taking that into account, it is easy to see that the degree of the polynomial is also linear with respect to the number of gradient steps.  Hence, learning from encrypted data in the way proposed here is feasible only for small datasets or for simple models such as linear models.

\section{Discussion}
  In Section~\ref{lemma:polyApprox} we have seen that from a theoretical point of view, it is possible to learn over encrypted data as well as to apply networks to encrypted data. However, in Section~\ref{sec:practical} we have seen that from practical consideration, some applications of crypto-nets are not feasible with the current construction. Therefore, it makes sense to study different use-cases and discuss the theoretical and practical implications of these scenarios.

  Doing inference with crypto-nets is a promising direction. In this scenario, the net is learned over plain data and is applied to encrypted data. For example, consider a dentist that may take X-ray images of suspect tooth and send them to be classified in a cloud service. With crypto-nets, the dentist can encrypt the image and send for evaluation without compromising the privacy of clients since not only the image is encrypted but also the prediction is only visible to the dentist and not to the owner of the predictive models. Another example includes a client that would like to apply for a loan from a bank. Currently, the client has to reveal private financial details to allow the bank to predict the risk associated with the loan. However, with crypto-nets, this can be done without revealing any private information. At the same time, inference over encrypted data is still slower than inference on plain data and hence suitable only in cases where latency and throughput are not major concerns.

  Learning with crypto-nets requires more detailed inspection. We propose three scenarios of learning with encrypted data. 
\begin{enumerate}
\item Assume that a sample is encrypted and the goal is to learn a model from this sample. As discussed in Section~\ref{sec:poly}, the theory suggests that this is possible. However, in practice this is feasible only if the sample is small or the network is shallow.

\item Assume that there are multiple samples, each encrypted with a different key, and the goal is to learn a model by aggregating these datasets. This is the case, for example, if multiple dentists store the medical records of their patients, each dentist using a different key. This scenario is not supported by the kind of homomorphic encryption we discussed so far. However, this could be addressed by secure multi-party computation \citep{du2001secure}. \cite{LTV12} presented a fully homomorphic encryption scheme that allows joint computation over data that was encrypted with different keys. The result would be owned by all parties that contributed data in the sense that decryption requires all data owners who contributed data to the computation to jointly decrypt.

\item Assume that a model has been trained using plain data but users may wish to adapt it to their data. Therefore, the model is already trained and the goal is to perform few gradient steps to fine tune it. This scenario is theoretically feasible and may be practical provided that the data size is small and that the network can be approximated by a polynomial of not-too-high degree.
\end{enumerate}

\section{Conclusion}
In this work we have presented crypto-nets: a way to learn and apply neural networks to encrypted data. We have discussed the theoretical aspects of learning and inferencing over encrypted data as well as the practical implications. We conjecture that for medical and financial applications, crypto-nets may be feasible for the inference stage and maybe even for some limited learning. Implementing crypto-nets require careful work both in the machine learning side and in the cryptology side and is subject of ongoing research.

\bibliography{cryptdnn}
\bibliographystyle{icml2014}

\end{document}